\documentclass[twoside]{article}

\usepackage[accepted]{aistats2024}
\usepackage[round]{natbib}

\usepackage[utf8]{inputenc} 
\usepackage[T1]{fontenc}    
\usepackage[hidelinks]{hyperref}       
\usepackage{url}            
\usepackage{booktabs}       
\usepackage{amsfonts}       
\usepackage{nicefrac}       
\usepackage{microtype}      
\usepackage{xcolor}         

\usepackage{graphicx}
\usepackage{float}
\usepackage{caption}
\usepackage{subfigure}
\usepackage{multirow}

\usepackage{algorithm}
\usepackage[noend]{algorithmic}

\usepackage{amsmath}
\usepackage{amssymb}
\usepackage{bbm}
\usepackage{mathtools}
\usepackage{amsthm}
\usepackage{xspace}

\usepackage[capitalize,noabbrev]{cleveref}

\newcommand{\ourmethod}{\textsc{jam-pgm}\xspace}
\newcommand{\pmw}{\textsc{pmw}^\text{Pub}\xspace}
\newcommand{\gem}{\textsc{gem}^\text{Pub}\xspace}
\newcommand{\gempriv}{\textsc{gem}\xspace}
\newcommand{\mwem}{\textsc{mwem}\xspace}
\newcommand{\mst}{\textsc{mst}\xspace}
\newcommand{\aim}{\textsc{aim}\xspace}
\newcommand{\pgm}{\textsc{Private-pgm}\xspace}
\newcommand{\univ}{\mathcal{X}}
\newcommand{\dpriv}{D}
\newcommand{\dpub}{D_{\text{pub}}} 
\newcommand{\dsynth}{S}
\newcommand{\npriv}{n}
\newcommand{\npub}{\Hat{n}}

\newcommand{\x}{\mathbf{x}}
\newcommand{\R}{\mathbb{R}}
\newcommand{\ind}{\mathbf{1}}

\newcommand{\jointcandsetdown}{W_\downarrow^{\pub,\priv}}
\newcommand{\priv}{{\small\textsf{priv}}}
\newcommand{\pub}{{\small\textsf{pub}}}
\newcommand{\model}{p_{\theta}}

\DeclareMathOperator*{\argmin}{arg\,min}
\newcommand{\normEuc}[1]{{\left\lVert #1 \right\rVert}}
\newcommand{\Score}{{\normalfont \text{Score}}}
\newcommand{\adult}{\textsc{ADULT} }
\newcommand{\salary}{\textsc{SALARY} }
\newcommand{\fire}{\textsc{FIRE} }
\newcommand{\taxi}{\textsc{NIST-TAXI} }
\newcommand{\titanic}{\textsc{TITANIC} }

\theoremstyle{plain}
\newtheorem{theorem}{Theorem}[section]
\newtheorem{proposition}[theorem]{Proposition}

\theoremstyle{definition}
\newtheorem{definition}[theorem]{Definition}

\theoremstyle{remark}

\begin{document}

%

%
\runningauthor{ Miguel Fuentes, Brett Mullins, Ryan McKenna, Gerome Miklau, Daniel Sheldon}

\twocolumn[

\aistatstitle{Joint Selection: Adaptively Incorporating Public Information for Private Synthetic Data}

\aistatsauthor{Miguel Fuentes\\University of Massachusetts Amherst \And  Brett Mullins\\University of Massachusetts Amherst\AND  Ryan McKenna\\Google Research\And Gerome Miklau\\Tumult Labs\And Daniel Sheldon\\University of Massachusetts Amherst}
\aistatsaddress{ } ]

\begin{abstract}
  Mechanisms for generating differentially private synthetic data based on marginals and graphical models have been successful in a wide range of settings. However, one limitation of these methods is their inability to incorporate public data. Initializing a data generating model by pre-training on public data has shown to improve the quality of synthetic data, but this technique is not applicable when model structure is not determined a priori. We develop the mechanism $\ourmethod$, which expands the adaptive measurements framework to jointly select between measuring public data and private data. This technique allows for public data to be included in a graphical-model-based mechanism. We show that $\ourmethod$ is able to outperform both publicly assisted and non publicly assisted synthetic data generation mechanisms even when the public data distribution is biased.
\end{abstract}

\section{INTRODUCTION}\label{sec:intro} 
A differentially private (DP) algorithm can extract valuable insights from sensitive data while provably limiting what can be learned about individuals~\citep{dwork2006calibrating}. However, when data is accessed repeatedly, the curator must track the accumulated privacy loss and add noise sufficient to protect the entire sequence of queries, which presents logistical challenges for many common settings, including exploratory data analysis. Therefore, there is considerable interest in releasing private synthetic datasets that can support a range of downstream analyses~\citep{charest2011can,chen2015differentially,zhang2017privbayes,xie2018differentially,jordon2018pate,zhang2018differentially,asghar2019differentially,bowen2020comparative,vietri2020new,ge2020kamino,mckenna2021winning}.

Much of the recent research for private synthetic data has its roots in the multiplicative-weights exponential mechanism (\mwem) for private query-answering~\citep{hardt2010simple}. These algorithms iteratively select workload queries to measure, then measure those queries (with noise) and use the results to update a model for the synthetic data. Finally they generate synthetic records from the model. The goal is to generate synthetic tabular data to accurately answer a set of workload queries. 
This general pattern has been referred to as the ``select-measure-generate'' paradigm \citep{mckenna2021winning}. Recent work formalizes this pattern as the ``adaptive measurements'' framework \citep{liu2021iterative}.  Within this framework, research has focused on different model representations, estimation methods, selection mechanisms, and computational efficiency~\citep{aydore2021differentially,zhang2020privsyn,liu2021iterative,liu2021leveraging,cai2021data,mckenna2022aim}. 

It is well known that public data, when available, can be used to boost accuracy of many differentially private algorithms~\citep{ji2013differential,alon2019limits,bassily2020private,amid2022public,zhou2020bypassing,bassily2018model,kairouz2021nearly,wang2020differentially,papernot2016semi}.
For example, public data can be used to pre-train models~\citep{liu2021leveraging,yu2021differentially}, select hyperparameters or model structure~\citep{mckenna2021winning}, or even answer some queries directly to save privacy budget. Public data may come from releases that were done before DP restrictions were instituted or they may even come in the form of previously released synthetic data. In general, if the public data and private data are ``similar enough'', performance gains can be very large. However, a significant issue is that one does not know in advance how similar the public and private data are. Some previous works assume public data and private data are from the same distribution, which is unrealistic \citep{bassily2020private, alon2019limits}.

We address the problem of differentially private query answering and synthetic data generation with the assistance of public data. Unlike prior work, which uses public data for pre-training or to determine the support of the data distribution~\citep{liu2021leveraging,liu2021iterative}, we integrate public data tightly into the selection process of the mechanism, which means we explicitly consider \emph{when} and \emph{how} to use public measurements as a proxy for private measurements. We focus primarily on workloads of marginals, and augment the selection step to allow measuring a marginal either from public or private data. Conceptually, measuring a private marginal is unbiased but requires privacy noise, while measuring a public marginal as an estimate of a private marginal is biased but noise-free: which of these is better depends on the data, workload, and privacy parameters. Therefore, the mechanism must (privately) decide which marginals to measure from public data and which marginals to measure from the private data. 

In this paper, we develop joint adaptive measurements with $\pgm$ ($\ourmethod$), the first approach that incorporates public data \emph{selection} into iterative methods for synthetic data.
$\ourmethod$ is an adaptive measurements approach that uses \pgm \citep{mckenna2019graphical} to model the data distribution. 
$\ourmethod$ privately selects from both public and private measurements with scores that estimate the error reduction expected from each measurement.
By automatically selecting which queries to answer with public data, $\ourmethod$ can benefit from public data that is accurate for some marginals but inaccurate for others. We show empirically that $\ourmethod$ can use public data to increase accuracy across a range of scenarios.
\section{BACKGROUND}\label{sec:background}
A private dataset $\dpriv$ is a collection of $\npriv$ records each containing potentially sensitive information about one individual. Each record $r = (r_1, ..., r_m)$ has $m$ attributes and each attribute $r_i$ takes a value from the discrete finite set $\univ_i$. Each record belongs to the data universe $\univ = \univ_1 \times \dots \times \univ_m$. We also consider a public dataset $\dpub$ that is a collection of $\npub$ records which are not subject to differential privacy constraints. We assume that $\dpub \in \univ^{\npub}$. 

\subsection{Differential Privacy}\label{sec:dp_background}
Differential privacy is a formal model of privacy that bounds the effect of any individual record on the output of a randomized algorithm. We say that datasets $D, D' \in \univ^{\npriv}$ are neighboring, denoted $D \sim D'$, if $D'$ can be obtained from $D$ by modifying the values of at most one record. Note that all differentially private mechanisms considered in this paper use this notion of the neighboring relation.

\begin{definition}[Differential privacy; DP] \label{defn:dp}
    A randomized mechanism $\mathcal{M} \colon \univ^{\npriv} \rightarrow \mathcal{R}$ is said to be $(\epsilon, \delta)$-DP if, for all neighboring datasets $D \sim D' \in \univ^{\npriv}$ and all measurable subsets $ S \subseteq \mathcal{R} $, we have $\Pr[\mathcal{M}(D) \in S] \leq e^\epsilon \cdot \Pr[\mathcal{M}(D') \in S] + \delta$.
\end{definition}

A useful alternative notion of differential privacy for analyzing the composition of mechanisms is zero-Concentrated Differential Privacy. 
\begin{definition}[Zero-concentrated differential privacy; zCDP] \label{defn:zcdp}
  A mechanism $\mathcal{M}$ satisfies $\rho$-zCDP if for any neighboring datasets $D \sim D'$ and for all $\gamma \in (1, \infty)$, it holds that $D_{\gamma}(\mathcal{M}(D) || \mathcal{M}(D')) \leq \rho \gamma$,
  where $D_{\gamma}$ is the $\gamma$-Renyi divergence between distributions $\mathcal{M}(D),\mathcal{M}(D')$.
\end{definition}

\begin{proposition}[zCDP to DP Conversion; \citealt{canonne2020discrete}] \label{prop:zcdp_conversion}
    If mechanism $\mathcal{M}$ satisfies $\rho$-zCDP, then, it satisfies $(\epsilon, \delta)$-DP for any $\epsilon > 0$ and $\delta = \min_{\alpha > 1} \frac{\exp((\alpha - 1)(\alpha\rho-\epsilon))}{\alpha - 1}\left(1 - \frac{1}{\alpha}\right)^\alpha$.
\end{proposition}

The synthetic data mechanisms considered in this paper utilize two building block mechanisms: the exponential mechanism for private selection and the Gaussian mechanism for private query measurement. To analyze the privacy of these mechanisms, an important quantity is sensitivity, the maximum change in function value on neighboring datasets. The $L_p$ sensitivity of a function $f$ is given by $\Delta_p(f) = \max_{D \sim D'}  \normEuc{f(D) - f(D')}_p$ where $f\colon \univ^{\npriv} \rightarrow \mathbb{R}^k$. 

\begin{proposition}[zCDP of Gaussian mechanism; \citealt{bun2016concentrated}]\label{defn:gaussian mech}
  Let $f: \mathcal{X}^n \rightarrow \mathbb{R}^k$ be a vector-valued function of the dataset. For dataset $D$, the Gaussian mechanism adds i.i.d.\ Gaussian noise to $f(D)$ with scale parameter $\sigma^2$ i.e., $\mathcal{M}(D) = f(D) + \sigma \Delta_2(f) \mathcal{N}(0, \mathbf{1})$,
  where $\mathbf{I}$ is the $k \times k$ identity matrix. The Gaussian Mechanism satisfies $\frac{1}{2\sigma^2}$-zCDP.
\end{proposition}

\begin{proposition}[zCDP of exponential mechanism; \citealt{cesar2021bounding}]\label{defn:exponential mech}
  Let $\epsilon > 0$ and $\Score: \mathcal{R} \times \univ^n \rightarrow \mathbb{R}$ be a function such that $\Score(r, D)$ is the quality score of candidate 
  $r \in \mathcal{R}$ for data set $D$. The exponential mechanism~\citep{mcsherry2007mechanism} outputs a candidate $r \in \mathcal{R}$ according to the following distribution:
  $\Pr[\mathcal{M}(D) = r] \propto \exp \big(\frac{\epsilon}{2 \Delta_1} \Score(r, D) \big)$,
  where $\Delta_1 = \sup_{r \in \mathcal{R}} \Delta_1(\Score(r, D))$.
  The exponential mechanism satisfies $\frac{\epsilon^2}{8}$-zCDP.
\end{proposition}
Later, we suppress the dependence on $D$ and write $\Score(r)$ when it is clear from context.

Our method adaptively selects which building block mechanisms to use and how much privacy budget to allocate at each round based on the output of previous rounds. Because of this, we use the following result for fully adaptive composition instead of more basic results for non-adaptive composition.
\begin{proposition}[Fully adaptive composition for zCDP; \citealt{whitehouse2022fully}] \label{prop:composition}
  Let $(\mathcal{M}_i)_{i \geq 1}^{\ell}$ be a sequence of adaptively chosen mechanisms and $(\rho_i)_{i = 1}^{\ell}$ be a sequence of adaptively chosen privacy parameters such that $\mathcal{M}_i$ satisfies $\rho_i$-zCDP for $1 \leq i \leq \ell$. Let $\mathcal{M}_{1:\ell}$ denote the mechanism releasing output $(\mathcal{M}_1, \ldots, \mathcal{M}_{\ell})$. If it is always the case that $\Sigma_{i = 1}^{\ell} \rho_i \leq \rho$ then the mechanism $\mathcal{M}_{1:\ell}$ satisfies $\rho$-zCDP.
\end{proposition}

\subsection{Marginals and Workloads}
A marginal is a collection of linear queries that captures low-dimensional structure of the data distribution. Given a subset of attributes $\tau \subseteq \{1, \ldots, m\}$, the marginal on $\tau$ is a histogram over the possible values the attributes in $\tau$ can take. 
\begin{definition}\label{defn:marginal}
    Let $\tau \subseteq \{1, \ldots, m\}$ be a subset of attributes, $\univ_\tau = \prod_{i \in \tau} \univ_i$, and $n_\tau = |\univ_\tau|$. Define $r_\tau = (r_i)_{i \in \tau}$, the restriction of record $r$ to $\tau$. The marginal on $\tau$ is a vector of counts $\x \in \mathbb{R}^{n_\tau}$ indexed by $t \in \univ_\tau$ such that $\x[t] = \sum_{r \in D} \ind[r_\tau = t]$. We denote the function that computes the marginal on $\tau$ as $q_\tau$.
\end{definition}

For a marginal query $q_\tau$, the $L_1$ sensitivity is $2$ and the $L_2$ sensitivity is $\sqrt{2}$. To verify this, observe that neighboring datasets differ on the values of attributes $\tau$ for at most one record, increasing a count in the histogram by one and decreasing another by one. The sensitively of measuring one of the linear queries contained in the marginal is the same as the sensitivity of measuring the entire marginal. This useful property is sometimes called ``the marginal trick" and it makes marginals an efficient class of measurements.

We define a workload $W$ as a set of linear queries. In this paper, we focus on the class of workloads consisting of marginal queries. Many synthetic data generation algorithms take a workload as an input so that the distribution of the output data can be tailored to the given workload. A workload can be general, such as the set of all marginal queries for three or fewer attributes, or it can be specific, where it may be designed with particular dataset or application in mind.

The goal of synthetic data generation is to create a mechanism that will minimize error for any given workload and any input dataset. We define a notion of error for a given workload.
\begin{definition}
\label{defn:workload_err}
   Let $W$ be a workload of marginals. The workload error of synthetic dataset $\dsynth$ on $W$ is defined as follows for a fixed private data set $\dpriv$:
   \begin{equation}\label{eq:workload_err}
       \text{Error}_W(\dsynth) = \frac{1}{\npriv|W|} \sum_{\tau \in W} \normEuc{q_\tau(\dpriv) - q_\tau(\dsynth)}_1
   \end{equation}
   We write $\text{Error}_\tau(S)$ if $W$ contains a single marginal $\tau$.
\end{definition}

\subsection{\pgm}\label{sec:priv_pgm_background}
\pgm \citep{mckenna2019graphical} is a general purpose and scalable approach to combining noisy measurements into a single representation of the data distribution from which records can be sampled. Mechanisms using \pgm such as \mst and \aim are among the state-of-the-art methods for differentially private synthetic data generation \citep{mckenna2021winning, mckenna2022aim}. Given some marginal queries $\tau_1, \dots , \tau_t$ and noisy query answers $y_1, \dots, y_t$ \pgm produces an estimate of the data distribution $\model$ where $\theta$ are the parameters of a probabilistic graphical model. In this paper, we will take for granted that $\model$ can be used to answer marginal queries $q_\tau(\model)$. \pgm solves an optimization problem to search the space of models $\theta\in\mathcal{P}$ for one that minimizes the loss function $\sum_{i=1}^t ||y_i - q_{\tau_i}(\model)||_2^2$.

Since \pgm represents the data distribution as a graphical model, it is capable of scaling effectively to high-dimensional settings. However, as noted in prior work \citep{mckenna2019graphical,mckenna2021winning,mckenna2021relaxed,mckenna2022aim,cai2021data} the complexity of \pgm depends crucially on the set of marginals that have been measured.  \pgm exposes a utility method Is-Tractable($\tau_1, \dots, \tau_t$) that determines if \pgm is capable of efficiently handling a given set of marginals.  Efficiency-aware mechanisms that use \pgm must utilize this function in order to prevent the mechanism from measuring marginals that \pgm cannot efficiently handle~\citep{cai2021data,mckenna2022aim}.
\section{JOINT ADAPTIVE MEASUREMENTS}\label{sec:JAM}
Given a private data set $\dpriv$, a public data set $\dpub$, a workload $W$, and a privacy budget $(\epsilon, \delta)$, our goal is to design a mechanism $\mathcal{M}$ that generates synthetic data $\dsynth$ to minimize $\text{Error}_{W}(S)$ while satisfying $(\epsilon, \delta)$-DP. To solve this task we follow a design pattern called ``adaptive measurements'' \citep{liu2021iterative} which can be applied to most private synthetic data algorithms. A general version of this algorithmic pattern is provided in \cref{alg:AM}. We augment this framework by extending the selection step and measurement step to include the public data, we also presenting a novel budgeting strategy we call ``frugal budgeting". We refer to this augmented framework as ``Joint Adaptive Measurements'' (JAM).

\begin{algorithm}[tb]
   \caption{Adaptive Measurements; \cite{liu2021iterative}}
   \label{alg:AM}
\begin{algorithmic}
   \STATE {\bfseries Input:} Private dataset $\dpriv$, Workload $W$, zCDP privacy budget $\rho$
   \STATE {\bfseries Output:} Synthetic dataset $\dsynth$
   \STATE Initialize model $p_{\theta_0}$ \\
   \FOR{$t=0$ {\bfseries to} $T - 1$}
        \STATE {\bfseries select} $\tau_t$ where model $p_{\theta_t}$ poorly approximates $\dpriv$.
        \STATE {\bfseries measure} let $y_i$ be a private measurement of the marginal $\tau_t$ made by a noise-addition mechanism.
        \STATE {\bfseries update} $p_{\theta_{t+1}}$ from noisy measured information.
            \begin{equation*}
                p_{\theta_{t+1}} \gets \argmin_{\model \in \mathcal{P}} L(\model; y_1, \dots, y_t)
            \end{equation*}
    \ENDFOR
   \STATE {\bfseries generate} synthetic data $\dsynth$ from $p_{\theta_{T}}$ (or some function of the iterates $p_{\theta_{0}}, ..., p_{\theta_{T}}$)
\end{algorithmic}
\end{algorithm}

\subsection{Public Proxy Estimator}
When trying to estimate the value of a marginal $\tau$ on the private data, we can use the public data as a proxy for the private data. To do this, we evaluate the marginal query $q_\tau$ on $\dpub$ and re-scale the result by $\frac{\npriv}{\npub}$ to account for the number of records in the data sets. 

The public proxy estimator does not depend on the private data at all, so it can be used without expending any privacy budget. From a privacy standpoint, the public proxy is ideal. As a statistical estimator, the public proxy is unusual because it is deterministic and biased. In contrast, measurements made with the Gaussian mechanism incur a privacy cost but they provide an unbiased estimator that has a predictable error profile based on the noise scale. 

Which estimator is expected to incur more error depends on the similarity between the public and private marginals, the dimensionality of the marginal vector, and the noise scale being used for the Gaussian mechanism. By carefully designing a score function that considers these factors, we use the exponential mechanism to select the best estimator for the situation.

\subsection{Public/Private Measurement}\label{sec:measurement}
For conciseness, we use one measurement function to capture the Gaussian mechanism and the public proxy estimator.
\begin{definition}
    \label{defn:measure}
    For public data $\dpub$ and private data $\dpriv$, the function $\text{Measure} : 2^{[m]}\times\{\priv, \pub\} \rightarrow \R^{n_q}$ takes a marginal query $q_\tau$, public/private indicator $i$, and noise parameter $\sigma^2$ and is given by
    \begin{equation}\label{eq:measurement_fn}
        \text{Measure}(\tau,i; \sigma^2) = \begin{cases}
            q_\tau(\dpriv)\! +\!\mathcal{N}(0, \sigma^2 \mathbf{I})\!&\text{$i$ is $\priv$} \\
            q_\tau(\dpub)\frac{\npriv}{\npub} &\text{$i$ is $\pub$}. \\
        \end{cases}
    \end{equation}
\end{definition}

\subsection{Joint Candidate Set}\label{sec:expanded_candidates}
Adaptive measurement algorithms use the exponential mechanism to select from a pool of candidate queries. To allow for the possibility of measuring a marginal query with either the public proxy estimator or the Gaussian mechanism, we include public and private versions of queries in our candidate set. We refer to this as a ``joint candidate set''.
\begin{definition}
    \label{defn:joint_cand_set}
    Given a set of candidate queries $W$, the joint extension $W^{\priv, \pub}$ of $W$ is $W \times \{\priv, \pub\}$.
\end{definition}
The indicator element $i \in \{\priv, \pub\}$ indicates whether a candidate corresponds to a measurement on the private data with the Gaussian mechanism or a measurement using the public proxy.

In \cref{alg:our_method} we utilize the joint extension of the downward closure candidate set. This allows lower-dimensional marginals to be selected and was first used in $\aim$ \citep{mckenna2022aim}.
\begin{definition}
    \label{defn:priv_cand_set_down}
    The joint downward closure candidate set for $W$ is $\jointcandsetdown = \{\tau' |\tau' \subseteq \tau, \tau \in W \} \times \{\priv, \pub\}$.
\end{definition}
To select from this set, we construct a score function that applies to both public and private candidates. In practice, we filter this set of candidates $\jointcandsetdown$ down to a set $C$ containing marginals that can be added to \pgm without rendering it intractable.

\subsection{Expected Improvement Score Function}\label{sec:expected_improv}
$\ourmethod$ uses a score function that quantifies the expected improvement in the model after making a measurement. This score function has the goal of simultaneously considering which queries are being poorly approximated by the model and which marginals could be accurately measured (on the public or private data). This idea has been explored in the private data setting by evaluating the expected error of the Gaussian mechanism \citep{mckenna2022aim}. When measuring a private marginal with the Gaussian mechanism, the expected error is given by $\sqrt{2/\pi}\sigma n_\tau$ for a given noise scale $\sigma$ and marginal size $n_\tau$.  When measuring a public marginal, the error is fixed and can be evaluated directly. Combining these gives the following function for measurement error:
\begin{definition}
    \label{defn:measurement_error}
    For fixed public data $\dpub$ and private data $\dpriv$, the predicted measurement error $\text{PredError} : 2^{[m]}\times\{\priv, \pub\} \rightarrow \R$ of a 
    marginal query $\tau$ and measurement indicator $i$ with noise scale $\sigma^2$ is
    \begin{equation}\label{eq:measurement_error}
        \text{PredError}(\tau,i; \sigma^2)\!=\!\begin{cases}
            \sqrt{2/\pi}\sigma n_\tau &i = \priv \\
            \normEuc{q_\tau(\dpriv)\!-\!q_\tau(\dpub)\frac{\npriv}{\npub}}_1\!&i=\pub \\
        \end{cases}
    \end{equation}
\end{definition}
Note that for each round of the algorithm, $\dpub$, $\dpriv$, and $\sigma^2$ will be fixed so the measurement error is given as a function of just $\tau$ and $i$. To estimate the improvement in the model after making a measurement, we assume that the model will match the value of the measurement on the marginal $q_\tau(\model) = \text{Measure}(\tau, i, \sigma^2)$. Under this assumption, the expected error in the next round of the algorithm would be $\text{PredError}(\tau,i; \sigma^2)$. So, the difference between the current error of the model and our estimate for the error after a measurement gives us the expected improvement score function:
\begin{definition}
    \label{defn:score_fn}
    For a fixed model $\model$, the expected improvement score function $\text{Score} : 2^{[m]}\times\{\priv, \pub\} \rightarrow \R$ of a marginal query $\tau$ and public/private indicator $i$ with noise parameter $\sigma^2$ is
    \begin{equation}\label{eq:score_fn}
        \text{Score}(\tau, i; \sigma^2) = \text{Error}_{\tau}(\model) - \text{PredError}(\tau,i; \sigma^2)
    \end{equation}
\end{definition}
The $L_1$ sensitivity of this score function is 4 because changing one record can change the value of $\text{Error}_{\model}(\tau)$ by 2 and the value of $\text{PredError}(\tau,i; \sigma^2)$ by 2.

\subsection{Frugal Budgeting}\label{sec:frugal_budgeting}
Besides accuracy, another advantage to measuring the public data is that it allows for what we call frugal budgeting. Each round, we allocate some portion of the remaining privacy budget to selection and some portion of the budget to measurement. If we select a public measurement, the budget allocated for measurement is unused. This allows us to take those budget savings and roll them over into the next round. Passing savings on to later rounds of the algorithm means that when a public measurement is selected, subsequent rounds will have higher budgets and lower noise scales.

\begin{algorithm}[tb]
   \caption{$\ourmethod$}
   \label{alg:our_method}
\begin{algorithmic}
\renewcommand{\baselinestretch}{1.2}\selectfont
   \STATE {\bfseries Input:} Public dataset $\dpub$, Private dataset $\dpriv$, Workload $W$
   \STATE {\bfseries Output:} Synthetic dataset $\dsynth$
   \STATE {\bfseries Hyperparameters:} Privacy parameter $\rho$, number of rounds $T$, select-measure split $\alpha$ \\[2pt]
   \STATE Initialize $\model \gets \text{Uniform}[\univ]$ \\[2pt]
   \FOR{$t=0$ {\bfseries to} $T - 1$}
        \STATE $\rho^t \gets \big(\rho - \sum_{s=0}^{t-1}\rho^s_\text{used}\big)/\left(T - t\right)$ \\[2pt]
        \STATE $\rho^t_\text{select}, \rho^t_\text{measure} \gets (1 - \alpha)\rho^t, \alpha\rho^t$
        \STATE $\sigma^2_t \gets 1 / \rho^t_\text{measure}$
        \STATE $C_t \gets \{\tau, i \in \jointcandsetdown | \text{Is-Tractable}(\tau, \tau_1, ..., \tau_{t - 1})\}$ \\[4pt]
        \STATE {\bfseries select} $\tau_t, i_t$ from $C_t$ using exponential mechanism with budget $\rho^t_\text{select}$ and $\text{Score}(\tau, i; \sigma_t^2)$ from \cref{eq:score_fn}\\[3pt]
        \renewcommand{\baselinestretch}{0.8}\selectfont
        \STATE {\bfseries measure} $\tau_t$ publicly or privately with measurement function from \cref{eq:measurement_fn}
            \begin{equation*}
                y_t = \text{Measure}(\tau_t,i_t; \sigma^2_t)
            \end{equation*}
        \STATE {\bfseries estimate} the data distribution using \pgm
            \begin{equation*}
                \model \gets \argmin_{\model' \in \mathcal{P}} \sum_{i = 1}^t \normEuc{y_i - q_{\tau_i}(\model')}_2^2 
            \end{equation*}
        \STATE $\rho_\text{used}^t \gets \rho_{\text{select}}^t + \mathbf{1}[i_t = \priv]\cdot\rho_{\text{measure}}^t$
    \ENDFOR
   \STATE {\bfseries generate} synthetic data $\dsynth$ from $\model$
\end{algorithmic}
\end{algorithm}

\subsection{Privacy Proof}\label{sec:privacy_proof_our_method}
The privacy analysis of \cref{alg:our_method} is an application of \crefrange{defn:gaussian mech}{prop:composition}.

\begin{theorem}
  For any number of rounds $T > 0$, budget split parameter $\alpha \in (0, 1)$, and privacy parameter $\rho > 0$, $\ourmethod$ satisfies $\rho$-zCDP. 
\end{theorem}

\begin{proof}
  By construction, each round of $\ourmethod$ satisfies $\rho_{\text{used}}^t$-zCDP, where $\rho_\text{used}^t = \rho_{\text{select}}^t + \ind[i_t = \priv]\cdot\rho_{\text{measure}}^t$ as defined in the algorithm: the selection step always satisfies $\rho_{\text{select}}^t$-zCDP, and the measurement step satisfies $\rho_{\text{measure}}^t$-zCDP if a private candidate is selected and $0$-zCDP if a public candidate is selected.
  Also by construction, $\rho_\text{used}^t < \rho^t = \nicefrac{(\rho - \sum_{s=0}^{t-1}\rho_\text{used}^s)}{(T-t)}$, so we have the invariant that $\sum_{s=0}^t \rho_{\text{used}}^s \leq \rho$. 
  Therefore, by \cref{prop:composition}, $\ourmethod$ satisfies $\rho$-zCDP.
\end{proof}
\section{PRIOR WORK}\label{sec:prior}
The field of DP synthetic data, also known as DP query release, has a rich history. The adaptive measurements framework provides useful language to describe a number of methods in a unified framework. Many of these algorithms draw inspiration from the \mwem algorithm \citep{hardt2010simple}, the first to iteratively refine a data model by selecting poorly approximated queries in each round. Since \mwem, various data models, selection criteria, and optimization procedures have been explored in the literature \citep{hardt2010simple, liu2021iterative, gaboardi2014dual, aydore2021differentially, mckenna2021winning, cai2021data, mckenna2022aim}. 

The first method to make use of public data for synthetic data generation was $\pmw$ \citep{liu2021leveraging}. This method is a version of the $\mwem$ algorithm that incorporates the public data in two ways: initialization and domain restriction. The original $\mwem$ algorithm represents the data distribution as a histogram where each entry corresponds to an element of the data universe $\univ$. It then uses an adaptive measurement strategy along with a multiplicative weights update rule to update that data distribution.  One problem with this algorithm is that the size of the histogram representation scales exponentially with the dimensionality of the data. To get around this, $\pmw$ restricts the histogram to elements of the data universe that are present in the public data. It also initializes that histogram to match the distribution of the public data. When the public and private data distributions are similar, this greatly improves performance. But when the distributions are very different, the domain restriction can make it impossible to generate good synthetic data.

This domain restriction problem was addressed in subsequent work, which introduced the $\gem$ method \citep{liu2021iterative}. This method represents the data distribution with a generator network and does not restrict the domain of the distribution. Instead, $\gem$ pre-trains on the public data to initialize the model. By doing this, $\gem$ realizes gains in performance without overly committing to public data, which may do a poor job of reflecting the private data distribution.

These works incorporate the public data into the initialization step of the adaptive measurements framework. This is fundamentally incompatible with methods that do not fix model structure a priori, such at \pgm. In a graphical model, the number and structure of parameters depend on the edges in the graph. In $\pgm$, edges are determined by the choice of marginal measurements so the model structure is not yet defined during the initialization step. In contrast, the joint selection framework incorporates the public data into the select and measure steps of the adaptive measurements framework.
\section{EXPERIMENTS}\label{sec:experiments}
In this section, we evaluate the performance of $\ourmethod$ against the baseline methods $\pmw$ \citep{liu2021leveraging} and $\gem$ \citep{liu2021iterative} which both incorporate public data. We also compare to \aim~\citep{mckenna2022aim} and $\gempriv$ \citep{liu2021iterative} which do not use public data.
Here we provide the key details of the experimental setup. Additional details on the compute environment and code are provided in the supplementary materials. Code for running our methods is provided on GitHub\footnote{\url{https://github.com/Miguel-Fuentes/JAM_AiStats/}}.

\begin{figure*}
    \centering
    \includegraphics[width=\textwidth]{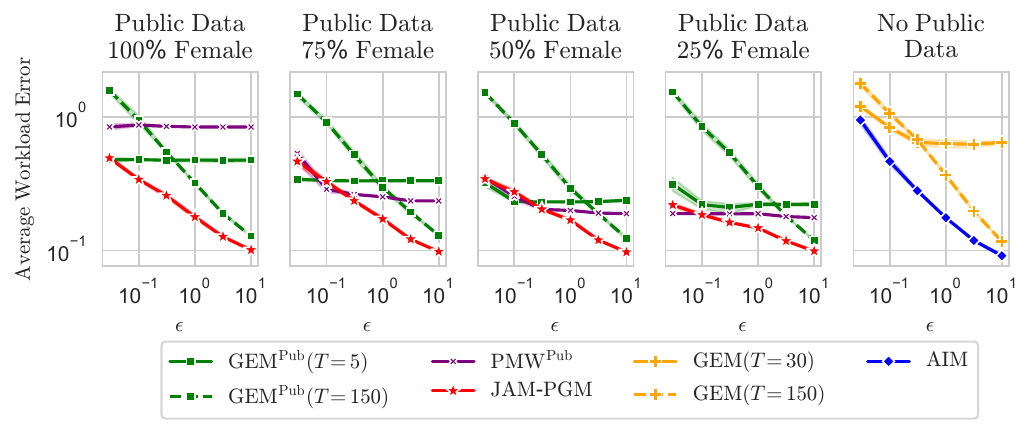}
    \captionof{figure}{Average workload error (workload of all 3 way marginals) for $\epsilon$ in \{0.03, 0.10, 0.31, 1.00, 3.16, 10.00\} and $\delta=1\times10^{-9}$ for the \adult data set. Private dataset consists of 25\% female records while public dataset consists of  100\%, 75\%, 50\%, and 25\% female records respectively. The plots are shown with the amount of public data bias decreasing from left to right.}
    \label{fig:adult}
\end{figure*}

\paragraph{Task Specification}
The task is to minimize workload error while satisfying $(\epsilon, \delta)$-DP. We evaluate on the privacy parameters $\epsilon \in \{0.03, 0.10, 0.31, 1.00, 3.16, 10.00\}$ and $\delta=1\times10^{-9}$. The workload provided was all 3-way marginals and 5 trials were run with 5 random seeds to provide standard error estimates.

\paragraph{Data}
Our first experiment is performed on the \adult dataset \citep{kohavi1996scaling}. We split the dataset into public and private datasets using a stratified sampling strategy to ensure that the public data distribution is different from the private data distribution. We sampled a private dataset with 32,384 records such that 25\% of those records are female. Then, we constructed various public datasets with 3,238 records each varying the percentage of female records. 

We also experiment with the following datasets: \salary \citep{hay2016principled}, \fire \citep{ridgeway2021challenge}, \taxi \citep{taxi}, and \titanic \citep{titanic}. We split the data by randomly selecting private records and iteratively adding them to the public dataset until the public data met an error target when used as a proxy for the private data. The \titanic dataset contains about one thousand records total while the other datasets contain hundreds of thousands of records. Because of this, the behaviour of the various methods on $\titanic$ differ from the behaviour on the other datasets and the results are discussed separately. Information regarding the number of attributes, number of private data records, number of public data records, average 3-way marginal size, and total domain size of the dataset used in our experiments can be seen in appendix \cref{tab:dataset_info}.

\paragraph{Hyperparameter Selection}
The main hyperparameter for each method is the number of rounds $T$ (except for $\aim$ which determines $T$ adaptively). For all methods, we conducted limited preliminary experiments to non-privately select a value or values for this parameter. For $\pmw$ and $\ourmethod$ it was relatively easy to find a single value that performed well across a range of epsilon values. However, this was not possible for $\gem$, which performed better with fewer rounds at low privacy budgets, and more rounds at high privacy budgets. 

We interpret this behavior as implicitly selecting ``how much'' to use the public data. With low epsilon (high privacy noise), it is beneficial to run for few rounds and remain close to the public data initialization, while with high epsilon (low privacy noise) it is beneficial to run for more rounds. To fairly represent the range of performance possible with $\gem$, we selected two values for $T$, the highest and lowest ones that performed reasonably in preliminary experiments. More details of hyperparameter selection appear in the appendix.

\subsection{Varying Public Data Bias}\label{sec:adult_results}
The results of the \adult experiment can be seen in \cref{fig:adult}. For clarity, the methods that do not use public data are shown on a separate subplot.

Across all levels of public data bias and almost all values of $\epsilon$, $\ourmethod$ achieved lower average workload error than the baseline methods. All of the methods that utilize the public data perform better when the public data distribution is more similar to the private data distribution. However, as the public data bias increases, the performance of $\ourmethod$ degrades more slowly than the other public data methods. Another change related to the public data bias is the relationship between the error curves of $\gem$ with $T=5$ and $\gem$ with $T=150$. When the public data are from the same distribution as the private data, it is beneficial to run $\gem$ for fewer rounds for a given value of $\epsilon$ but as the bias increases it becomes beneficial to run $\gem$ for more rounds for a given value of $\epsilon$.

In the setting where the public data consists of only female records, we see that $\pmw$ has much worse performance than in the settings where the public data includes male records. This highlights the risk described in prior work \citep{liu2021leveraging, liu2021iterative}: if the public and private data are not sufficiently compatible, $\pmw$ will be unable to produce a good model of the private data regardless of privacy budget.

Note that while the performance of $\ourmethod$ and $\aim$ are similar for high values of $\epsilon$, the average workload error of $\aim$ is slightly lower than that of $\ourmethod$ for $\epsilon = 10$ across all levels of public data bias. 

\subsection{Small Data Set}\label{sec:titanic}
The results of the \titanic experiment can be seen in \cref{fig:titanic}. In this setting, where the private data set consists of only one thousand records, there is a large separation between the methods that utilize public data and those that do not up to $\epsilon = 3.16$. $\pmw$ performs best for $\epsilon \leq 0.31$, but it does not improve as the privacy budget increases. For larger values of $\epsilon$, $\ourmethod$ and $\gem$ have very similar average workload error.

Notice that for most values of $\epsilon$ in this experiment, the number of rounds that performs best for the $\gem$ algorithm is 1. Increasing the number of rounds to 5 significantly decreases performance for $\epsilon < 3.16$. This suggests that in this setting, the model initialization is much more informative than the noisy measurements. In cases like these with few records and low privacy budget, it may be advantageous to not make any measurements at all and simply rely on the public data instead.

\begin{figure}
    \centering
    \includegraphics[width=\columnwidth]{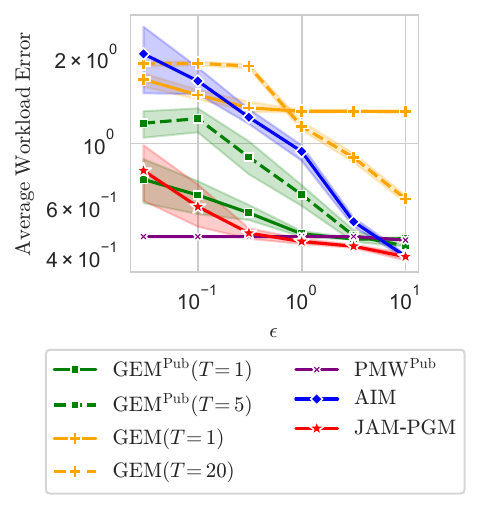}
    \caption{Average workload error (workload of all 3-way marginals) for $\epsilon$ in \{0.03, 0.10, 0.31, 1.00, 3.16, 10.00\} and $\delta=1\times10^{-9}$ for \titanic data set.}
    \label{fig:titanic}
\end{figure}

\begin{figure*}
    \centering
    \includegraphics[width=\textwidth]{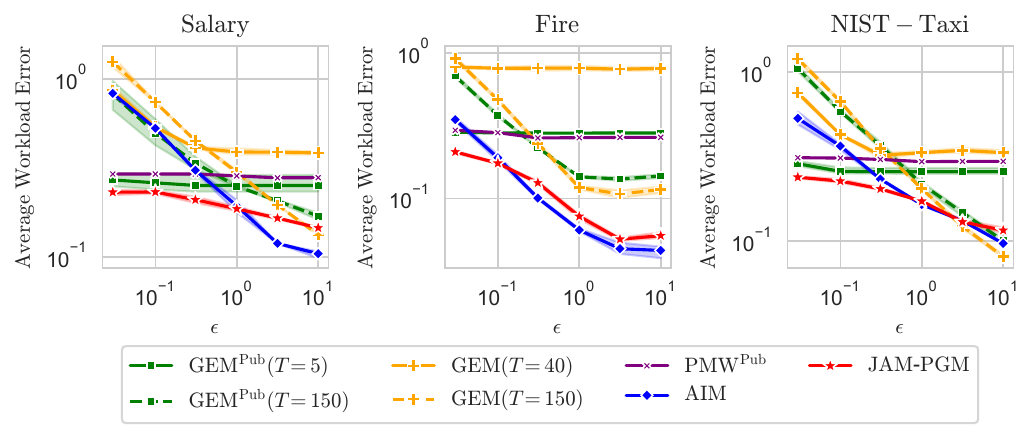}
    \caption{Average workload error (workload of all 3-way marginals) for $\epsilon$ in \{0.03, 0.10, 0.31, 1.00, 3.16, 10.00\} and $\delta=1\times10^{-9}$ for the \salary, \fire, and \taxi datasets.}
    \label{fig:multiple}
\end{figure*}

\subsection{Larger Data Sets}\label{sec:fire_salary_taxi}
The results of the \salary, \fire, and \taxi experiments can be seen in \cref{fig:multiple}. In these experiments, $\ourmethod$ outperforms the other public-data utilizing mechanisms for all values of epsilon (except for $\epsilon = 10$ on the \taxi dataset). However, none of the methods that incorporate public data are optimal in these settings. $\aim$ performs best in the \salary experiment for $\epsilon > 1$ and in the \fire experiment for $\epsilon > 0.1$ while $\gempriv$ performs best in the \taxi experiment for $\epsilon \geq 3.16$. We might hope that having access to additional information would never hurt the performance of these synthetic data mechanisms but these experiments show that this is not the case.

\section{DISCUSSION}\label{sec:limitations_and_discussion}
Our development of the ``joint adaptive measurements'' framework incorporates public data into the {\em selection} and {\em measurement} steps of the ``adaptive measurements'' framework. This expands the design space for public-data assisted DP algorithms; using this technique, we develop $\ourmethod$, which is able to use joint selection and \pgm in order to effectively select a combination of public measurements and private measurements and utilize them to generate high quality DP synthetic data.

\paragraph{Limitations}
One limitation of the joint selection framework is the lack of public data error estimates, a benefit of noise-addition mechanisms is that the distribution of noise is known and can be released publicly. \pgm \citep{mckenna2019graphical} uses this information to weight the measurements provided in its optimization objective based on the noise scale.  \aim also uses this information to perform budget annealing which dynamically determines the number of rounds \citep{mckenna2022aim}. The error of a public measurement is sensitive information and budget would need to be spent to measure it. Because of this, $\ourmethod$ gives equal weight to all measurements in the \pgm loss function and uses a fixed number of rounds instead of a budget annealing strategy. Another limitation, inherited from \pgm, is that $\ourmethod$ can not be applied to attributes with continuous values without discretization this limitation is common in this area and also applies to the other methods we compare against.

\paragraph{Suboptimality of Publicly-Assisted Methods:}
Ideally, utilizing public data would strictly improve performance compared to not utilizing public data. Our experiments show that no public-data assisted mechanism currently achieves this. On all of the data sets we tested except for \titanic, there were some values of $\epsilon$ for which none of the methods that incorporate public data outperform the baselines that do not use public data. This indicates that all of the techniques that have been developed to incorporate public data (domain restriction, pre-training, and joint selection) have some risk of performing worse than baseline methods that ignore public data. On the other hand, the public techniques clearly provide sizeable boosts in performance when the distribution of the public data closely matches the distribution or when $\epsilon$ is small. Navigating this trade-off in practice may be difficult for practitioners.

\paragraph{Benefits of $\ourmethod$}
Across all our experiments, $\ourmethod$ achieved lower average workload error than the public data baselines for most values of $\epsilon$. In addition, the adaptive nature of $\ourmethod$ alleviates some of the difficulties associated with the use of public data in practice. Users will not know a priori exactly how similar their public data and private data distributions are, so they may be hesitant to use $\pmw$ because of the risks associated with domain restriction. Similarly but to a lesser extent, selecting the number of rounds for $\gem$ may be challenging because there is no setting for this parameter that works well across values of $\epsilon$ and for a given $\epsilon$ the optimal number of rounds depends on the data set. $\ourmethod$ is a safe choice for DP synthetic data generation. Unlike $\pmw$, it does not have the risk associated with restricting the domain of the model based on public data. Additionally, there exist settings for the number of rounds parameter $T$ that work well across a range of privacy budgets which may make it easier to use in practice than $\gem$. 

\paragraph{Bounded vs Unbounded DP}
We adopt bounded DP, defining dataset neighbors as those differing by swapping a single record. Under this regime, the number of records in the private dataset is considered public information. This is useful when scaling public marginals to match the number of records in the private dataset. In contrast, unbounded DP considers neighbors based on adding or removing a single record, necessitating private estimation of the number of private records. While estimating this number is straightforward, for simplicity, we opt for bounded DP, omitting the need for a separate estimation step.

\section*{Acknowledgements}
We would like to thank Cecilia Ferrando and Javier Burroni for their helpful comments on earlier drafts of this paper. This material is based upon work supported by the National Science Foundation under Grant No. 1749854.

\bibliography{synth}
\newpage
\section*{Checklist}
 \begin{enumerate}
 \item For all models and algorithms presented, check if you include:
 \begin{enumerate}
   \item A clear description of the mathematical setting, assumptions, algorithm, and/or model. [Yes]
   \item An analysis of the properties and complexity (time, space, sample size) of any algorithm. [Yes]
   \item (Optional) Anonymized source code, with specification of all dependencies, including external libraries. [Yes]
 \end{enumerate}
 
 \item For any theoretical claim, check if you include:
 \begin{enumerate}
   \item Statements of the full set of assumptions of all theoretical results. [Yes]
   \item Complete proofs of all theoretical results. [Yes]
   \item Clear explanations of any assumptions. [Yes]     
 \end{enumerate}

 \item For all figures and tables that present empirical results, check if you include:
 \begin{enumerate}
   \item The code, data, and instructions needed to reproduce the main experimental results (either in the supplemental material or as a URL). [Yes]
   \item All the training details (e.g., data splits, hyperparameters, how they were chosen). [Yes]
         \item A clear definition of the specific measure or statistics and error bars (e.g., with respect to the random seed after running experiments multiple times). [Yes]
         \item A description of the computing infrastructure used. (e.g., type of GPUs, internal cluster, or cloud provider). [Yes]
 \end{enumerate}

 \item If you are using existing assets (e.g., code, data, models) or curating/releasing new assets, check if you include:
 \begin{enumerate}
   \item Citations of the creator If your work uses existing assets. [Yes]
   \item The license information of the assets, if applicable. [Yes]
   \item New assets either in the supplemental material or as a URL, if applicable. [Yes]
   \item Information about consent from data providers/curators. [Yes]
   \item Discussion of sensible content if applicable, e.g., personally identifiable information or offensive content. [Not Applicable]
 \end{enumerate}

 \item If you used crowdsourcing or conducted research with human subjects, check if you include:
 \begin{enumerate}
   \item The full text of instructions given to participants and screenshots. [Not Applicable]
   \item Descriptions of potential participant risks, with links to Institutional Review Board (IRB) approvals if applicable. Not Applicable]
   \item The estimated hourly wage paid to participants and the total amount spent on participant compensation. [Not Applicable]
 \end{enumerate}

 \end{enumerate}

%

%

\onecolumn
\setcounter{section}{0}
\aistatstitle{Appendix}

\section{Data}
Here we describe the preprocessing and public/private splitting strategies applied to the \adult, \fire, \salary, \taxi, and \titanic data sets. We also provide a table showing additional information about those data sets.

\subsection{Prepossessing}
In order to be consistent with prior work, we follow the preprocessing steps described in \citep{mckenna2022aim}. The first step is attribute selection, again following the lead of \cite{mckenna2022aim}, we keep all the attributes from the \adult, \salary, and \titanic datasets but we remove the 15 attributes relating to incident times in the \fire data set because after discretization, they contain redundant information. Next, we identify the domain of each data set. Normally a data provider would make this information public separately from the records in the data set, but this was not the case with these data sets. Therefore, we determine the domain by looking at the records in the data set. For each categorical attribute, we list the set of observed values and treat that as the set of possible values for that attribute. For each numerical attribute, we record the minimum and maximum observed value for that attribute. Finally, we discretize the continuous attributes by 32 equal-width bins, using the min/max values determined in the prior step.

\subsection{Public/Private Split}
Here, we describe how we split the data into a private data sets and public data sets. The public data that was generated by variable stratification is intentionally biased with respect to the private data, because of this, the public proxy estimator error will not approach 0 as the number of public records increases. To generate public data without variable stratification we sampled very small public data sets so that the sampling error would be significant enough that the algorithms would still need to access the private data.

\subsubsection{Variable Stratification}
We performed variable stratification on the\adult dataset based on the sex attribute \citep{kohavi1996scaling}. To do this, we split male records from the female records, then we sampled a private dataset with 32,384 records such that 25\% of those records are female. The next step was to sample four public datasets with 3,238 records each such that 100\%, 75\%, 50\%, and 25\% records were female. To generate a public data set with $p$\% female records, we would randomly sample $0.01 * p * 3,238$ of the remaining female records and $0.01 * (1-p) * 3,238$ of the remaining male records.

\subsubsection{Public Error Targeting}
To split the \salary, \fire, \taxi, and \titanic data sets, we set a target for public proxy error. Public proxy error is given by $\frac{1}{|W|}\sum_{\tau\in W}||q_\tau(\dpriv) - \frac{\npriv}{\npub} q_\tau(\dpub)||_1$. To create a split with the desired public proxy error, we started with one record in public data set. Then, we evaluated the public proxy error, if it was greater than the public error target we would double the size of the public data set by removing records from the private data set. We repeated this process until the public data met the error target. This doubling process gave a rough estimate for the number of public records but would give public error that was too far below the target. To resolve this, we would round down the number or records to the nearest hundred or thousand until we got closer to the error target. The original error target for all four data sets was 0.3. This target was reached for the \salary, \fire, and \taxi data sets. The \titanic data set is smaller than the others so we could not find a split that achieved the $0.3$ target; because of this, we changed the target to $0.5$.

\subsubsection{Additional Dataset Information}
\begin{table*}
    \centering
    \begin{tabular}{|c|c|c|c|c|c|}
        \hline
        Dataset Name & Columns & $\npriv$ & $\npub$ & Avg 3-way Marginal Size & Total Domain Size\\
        \hline
        \adult & 15 & 32,384 & 3,238 & $5.82\times 10^3$ & $4.09\times 10^{16}$\\
        \hline
        \titanic & 9 & 1,004 & 300 & $2.07\times 10^3$ & $8.92\times 10^{7}$\\
        \hline
        \salary & 9 & 131,727 & 4,000 & $1.64\times 10^5$ & $1.34\times 10^{13}$\\
        \hline
        \fire & 15 & 304,249 & 870 & $3.50\times 10^3$ & $4.21\times 10^{15}$\\
        \hline
        \taxi & 10 & 223,551 & 2,500 & $2.76\times 10^4$ & $1.87\times 10^{13}$\\
        \hline
    \end{tabular}
    \caption{Additional information for the datasets used in our experiments: Number of columns, number of private data records, nuber of public data records, average 3-way marginal size, and total domain size}
    \label{tab:dataset_info}
\end{table*}

\section{Additional Experimental Details}
\paragraph{Code:}
To run \aim and use \pgm in the context of \ourmethod, we used the code provided by the authors at \href{https://github.com/ryan112358/private-pgm}{https://github.com/ryan112358/private-pgm}. The \aim code assumed unbounded DP, whereas $\ourmethod$ assumes bounded DP. In order to compare fairly, we modified the sensitivity values in the \aim code to use bounded DP sensitivities. The code to run $\ourmethod$ and the version of $\aim$ with bounded DP will be available publicly at the time of publication. To run $\gem$, we use the code provided by the authors at \href{https://github.com/terranceliu/dp-query-release}{https://github.com/terranceliu/dp-query-release}. To run $\pmw$, we use the code provided by the authors at \href{https://github.com/terranceliu/pmw-pub}{https://github.com/terranceliu/pmw-pub}. 

\paragraph{Compute Environment:}
All experiments were run on internal compute clusters. The $\gem$ code is compatible with GPU acceleration, so it was run with on various types of NVIDIA GPUs, mostly (GeForce GTX TITAN X GPUs).

\paragraph{Size Limit:} The \pgm based methods require setting a size limit. We used a size limit of 80MB for both \aim and $\ourmethod$ across all experiments. When applying a size limit to a \pgm based method.  Our experiments ran with an adaptive size limit; during each round $t$ a size limit $s^t$ is determined based on the amount of privacy budget used so far $s^t = s_\text{total} (\frac{\rho_\text{used}^{t-1}}{\rho})$. This adaptive size limit ensures that the model grows slowly over the course of the rounds, which also speeds up inference in the early rounds because the model is guaranteed to be smaller.

\section{Hyperparameters}
The effect of changing the number of rounds hyperparameter $T$ for $\ourmethod$, $\gem$, $\gempriv$, and $\pmw$ are shown in visually in Figures 1-10. $\ourmethod$ has one other hyperparameter that was not searched over and that was $\alpha=0.8$. $\gem$ and $\gempriv$ have several other hyperparameters that were not searched over in \citep{liu2021iterative} so we followed their lead and kept those constant. Those hyperparameters were hidden layer sizes of $(512, 1024, 1024)$, learning rate $0.0001$, $B=1000$, and $\alpha=0.67$.

\subsection{$\gem$ and $\gempriv$ Round Sensitivity}
Across many of our experiments, we see that the version of $\gempriv$ that does not incorporate public data does have a relationship between the privacy budget and the optimal number of rounds. However, running for a relatively high number of rounds tends to lead to relatively good performance across all privacy budgets. The same is not true for $\gem$, once public data is incorporated the gap between high numbers of rounds and low numbers of rounds widens when the privacy budget is small. This is especially true when the public data are very representative of the private data.

\subsubsection{\adult Data Set}
The sensitivity in performance of $\gem$ with respect to the number of rounds hyperparameter on the \adult datasets are given in \cref{fig:adult_hyper_gempub}. The sensitivity in performance of $\gempriv$ with respect to the number of rounds hyperparameter on the private \adult data is given in \cref{fig:adult_hyper_gem}.

\begin{figure}[H]
    \centering
    \includegraphics[scale=1]{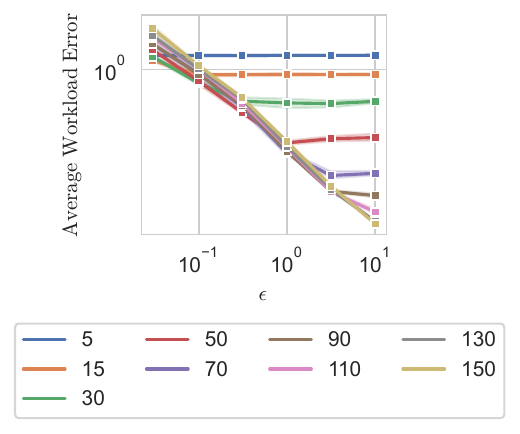}
    \caption{Average workload error for $\gempriv$ on the private \adult data set. The colors indicate the setting of the rounds hyperparameter $T$.}
    \label{fig:adult_hyper_gem}
\end{figure}

\begin{figure}[H]
    \centering
    \includegraphics[scale=1]{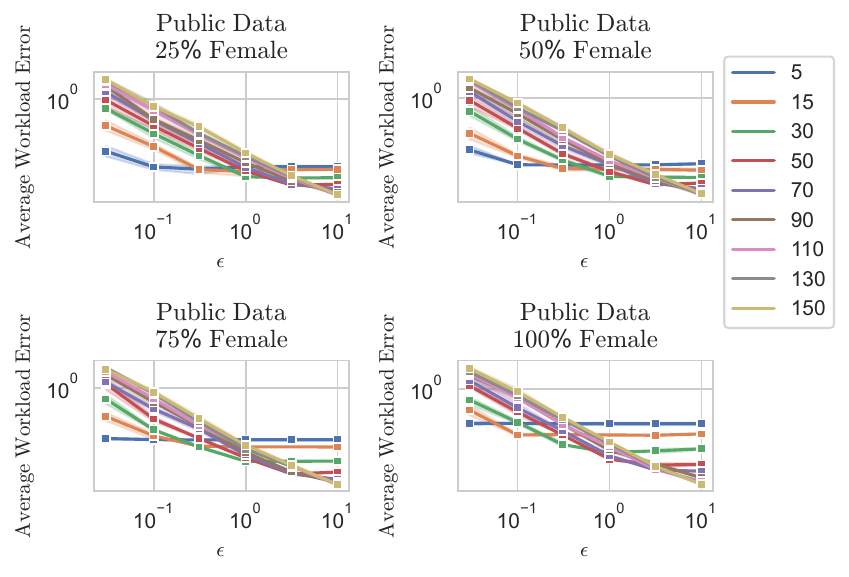}
    \caption{Average workload error for $\gem$ on the \adult data sets. The colors indicate the setting of the rounds hyperparameter $T$.}
    \label{fig:adult_hyper_gempub}
\end{figure}

\subsubsection{\titanic Data Set}
The sensitivity in performance of $\gem$ and $\gempriv$ with respect to the number of rounds hyperparameter on the \titanic data set are given in \cref{fig:titanic_hyper_gempub}.
\begin{figure}[H]
    \centering
    \includegraphics[scale=1]{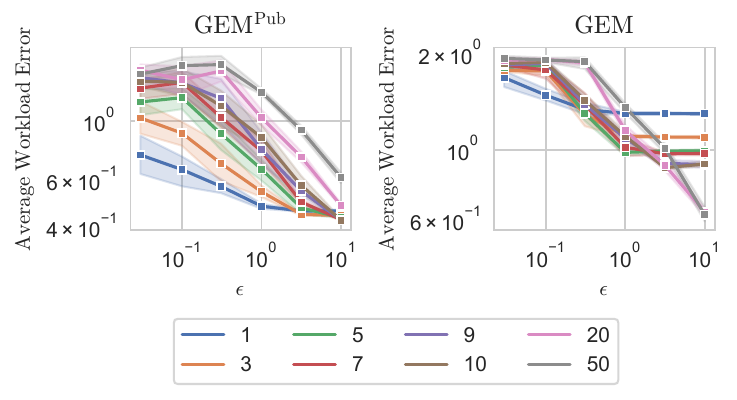}
    \caption{Average workload error for $\gem$ on the \titanic data set. The colors indicate the setting of the rounds hyperparameter $T$.}
    \label{fig:titanic_hyper_gempub}
\end{figure}

\subsubsection{Large Data Sets}
The sensitivity in performance of $\gem$ and $\gempriv$ with respect to the number of rounds hyperparameter on the \salary, \fire, and \taxi data sets are given in \cref{fig:multiple_hyper_gempub}.
\begin{figure}[H]
    \centering
    \includegraphics[scale=1]{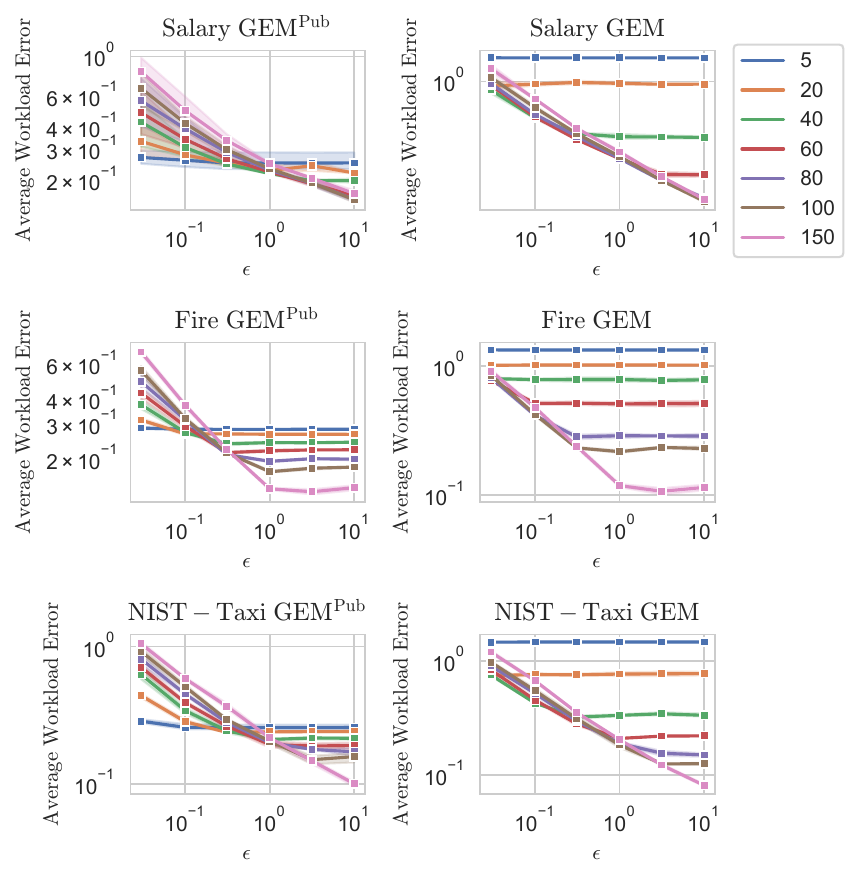}
    \caption{Average workload error for $\gem$ on the \salary, \fire, and \taxi data sets. The colors indicate the setting of the rounds hyperparameter $T$.}
    \label{fig:multiple_hyper_gempub}
\end{figure}

\subsection{$\pmw$ Round Sensitivity}
The average workload error of $\pmw$ is very insensitive to the number of rounds, notice that the y axes of the figures in this section have a very small range. 

\subsubsection{\adult Data Set}
The sensitivity in performance of $\pmw$ with respect to the number of rounds hyperparameter on the \adult data sets are given in \cref{fig:adult_hyper_pmw}.
\begin{figure}[H]
    \centering
    \includegraphics[scale=1]{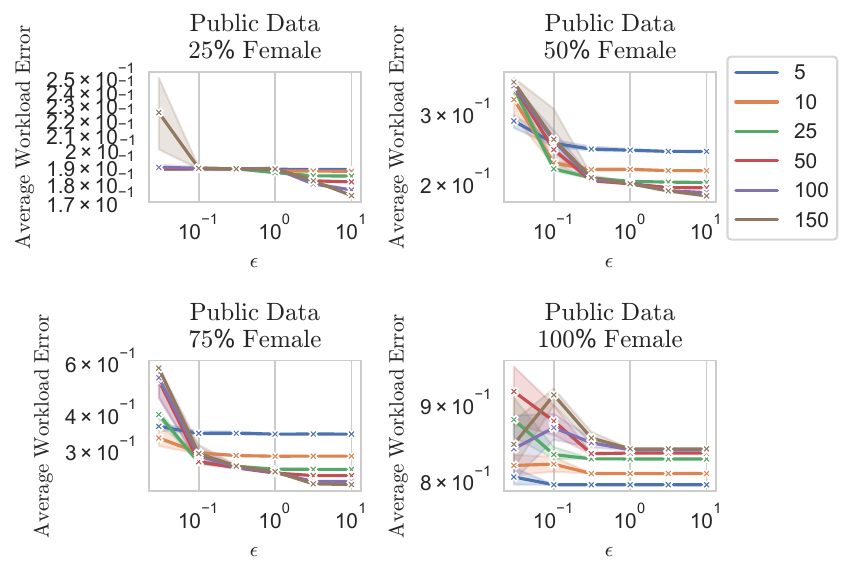}
    \caption{Average workload error for $\pmw$ on the \adult data sets. The colors indicate the setting of the rounds hyperparameter $T$.}
    \label{fig:adult_hyper_pmw}
\end{figure}

\subsubsection{\titanic Data Set}
The sensitivity in performance of $\pmw$ with respect to the number of rounds hyperparameter on the \titanic data set is given in \cref{fig:titanic_hyper_pmw}.
\begin{figure}[H]
    \centering
    \includegraphics[scale=1]{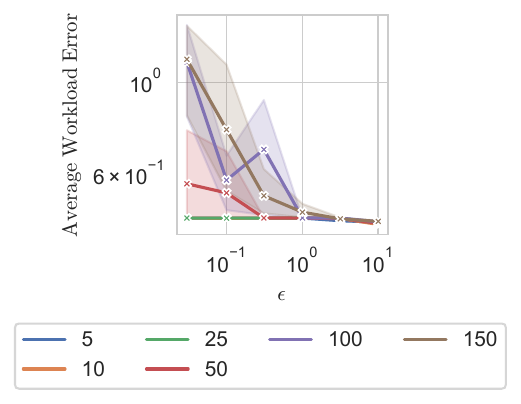}
    \caption{Average workload error for $\pmw$ on the \titanic data set. The colors indicate the setting of the rounds hyperparameter $T$.}
    \label{fig:titanic_hyper_pmw}
\end{figure}

\subsubsection{Large Data Sets}
The sensitivity in performance of $\pmw$ with respect to the number of rounds hyperparameter on the \salary, \fire, and \taxi data sets are given in \cref{fig:multiple_hyper_pmw}.
\begin{figure}[H]
    \centering
    \includegraphics[scale=1]{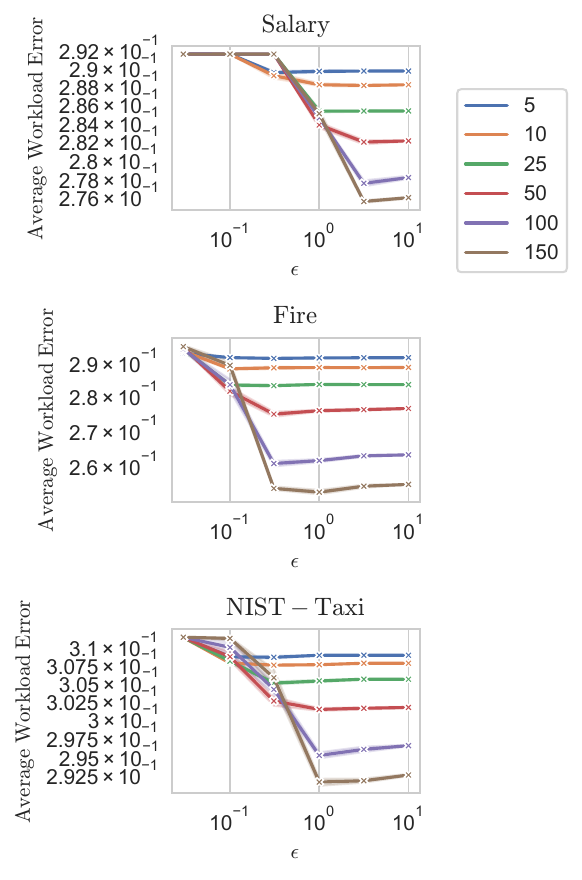}
    \caption{Average workload error for $\pmw$ on the \salary, \fire, and \taxi data sets. The colors indicate the setting of the rounds hyperparameter $T$.}
    \label{fig:multiple_hyper_pmw}
\end{figure}

\subsection{$\ourmethod$ Round Sensitivity}
The performance of $\ourmethod$ is not very sensitive to the number of rounds, it seems that if the number of rounds is too low performance suffers but as long as the number of rounds is high enough the performance does not vary much.

\subsubsection{\adult Data Set}
The sensitivity in performance of $\ourmethod$ with respect to the number of rounds hyperparameter on the \adult data sets are given in \cref{fig:adult_hyper_jam}.
\begin{figure}[H]
    \centering
    \includegraphics[scale=1]{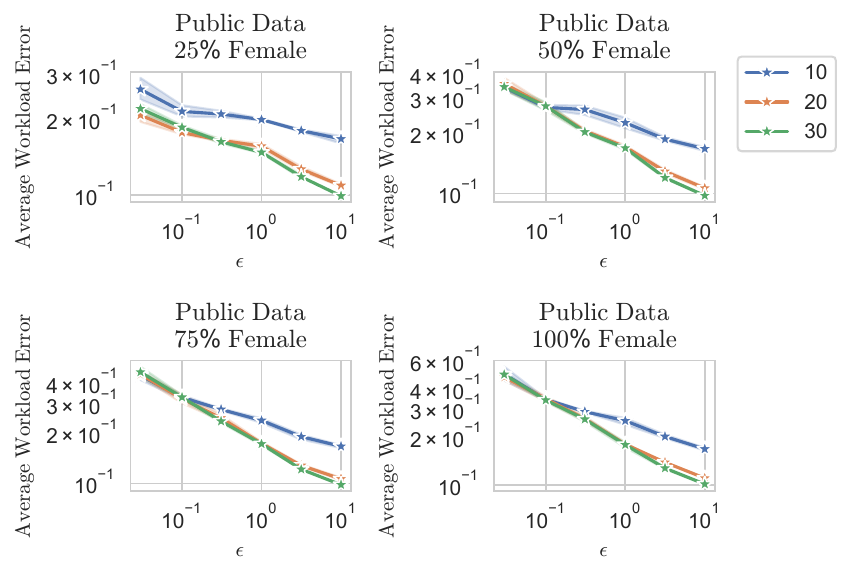}
    \caption{Average workload error for $\ourmethod$ on the \adult data sets. The colors indicate the setting of the rounds hyperparameter $T$.}
    \label{fig:adult_hyper_jam}
\end{figure}

\subsubsection{\titanic Data Set}
The sensitivity in performance of $\ourmethod$ with respect to the number of rounds hyperparameter on the \titanic data set is given in \cref{fig:titanic_hyper_jam}.
\begin{figure}[H]
    \centering
    \includegraphics[scale=1]{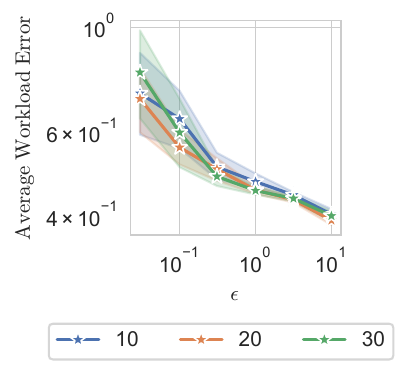}
    \caption{Average workload error for $\ourmethod$ on the \titanic data set. The colors indicate the setting of the rounds hyperparameter $T$.}
    \label{fig:titanic_hyper_jam}
\end{figure}

\subsubsection{Large Data Sets}
The sensitivity in performance of $\ourmethod$ with respect to the number of rounds hyperparameter on the \salary, \fire, and \taxi data sets are given in \cref{fig:multiple_hyper_jam}.
\begin{figure}[H]
    \centering
    \includegraphics[scale=1]{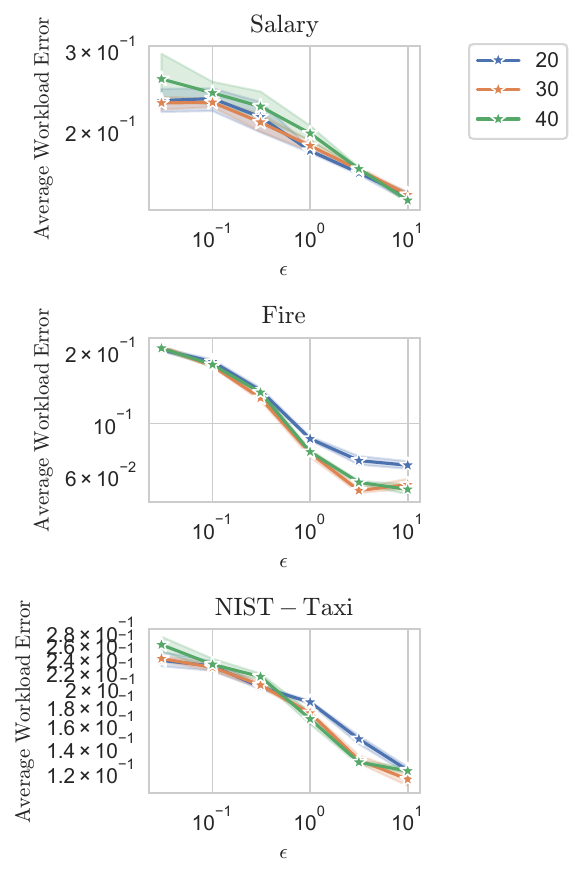}
    \caption{Average workload error for $\ourmethod$ on the \salary, \fire, and \taxi data sets. The colors indicate the setting of the rounds hyperparameter $T$.}
    \label{fig:multiple_hyper_jam}
\end{figure}

\end{document}